\documentclass[letterpaper, 10 pt, conference]{ieeeconf}
\IEEEoverridecommandlockouts
\let\labelindent\relax

\usepackage{amsmath}
\usepackage{graphics}
\usepackage{amssymb}
\usepackage{mathtools}
\usepackage{xcolor}
\usepackage{cite}
\usepackage{bbm}
\usepackage{pifont}
\usepackage[font={small}]{caption}
\usepackage{subcaption}
\usepackage{algorithm}
\usepackage{algorithmic}
\usepackage{amsthm}
\usepackage{siunitx}
\usepackage{enumitem}
\usepackage{hyperref}
\usepackage[textsize=footnotesize]{todonotes}
\usepackage{tikz}
\usetikzlibrary{positioning}
\usetikzlibrary{shapes.geometric}

\makeatletter
\newcommand\fs@betterruled{%
  \def\@fs@cfont{\bfseries}\let\@fs@capt\floatc@ruled
  \def\@fs@pre{\vspace{1pt}\hrule height0.7pt depth0pt \vspace{1pt}}%
  \def\@fs@post{\hrule\relax\vspace{-15pt}}%
  \def\@fs@mid{\vspace{1pt}\hrule}%
  \let\@fs@iftopcapt\iftrue}
\floatstyle{betterruled}
\restylefloat{algorithm}
\makeatother
\algsetup{indent=4pt}

\DeclareMathOperator*{\argmax}{arg\hspace{-1pt}\,max}

\newtheorem{theorem}{Theorem}

\newtheorem{lemma}{Lemma}

\newtheorem{definition}{Definition}
\newtheorem{problem}{Problem}

\newtheorem{example}{Example}

\title{\LARGE \bf
Model-Free Learning of Safe yet Effective Controllers
}

\author{Alper Kamil Bozkurt, Yu Wang, and Miroslav Pajic%
\thanks{This work is sponsored in part by the ONR under agreements N00014-17-1-2504, N00014-20-1-2745 and N00014-18-1-2374, AFOSR award number FA9550-19-1-0169, and the NSF CNS-1652544 grant.}%
\thanks{Alper Kamil Bozkurt, Yu Wang, and Miroslav~Pajic are with Duke University, Durham, NC 27708, USA, {\tt\small \{alper.bozkurt, yu.wang094,  miroslav.pajic\}@duke.edu}}}

\begin{document}

\maketitle
\thispagestyle{empty}
\pagestyle{empty}

\begin{abstract}
We study the problem of learning \emph{safe} control policies that are also \emph{effective}; i.e., maximizing the probability of satisfying a linear temporal logic (LTL) specification of a task, and the discounted reward capturing the (classic) control performance. We consider unknown environments 
modeled as Markov decision processes. We propose a model-free reinforcement learning algorithm that learns a policy that first maximizes the probability of ensuring safety, then the probability of satisfying the given LTL specification and lastly, the sum of discounted Quality of Control rewards. Finally, we illustrate applicability of our RL-based approach. %
\end{abstract}

\section{Introduction}
\label{sec:intro}
Many sequential decision making problems involve multiple objectives. For example, a robotic task might require surveillance without hitting the walls while minimizing the energy consumption. %
Quality of Control (QoC) can be often inherently expressed via scalar reward signals where the overall performance of a control policy (i.e., QoC) is measured by the sum of discounted collected rewards (returns)~\cite{sutton2018}. 
Although these return maximization objectives are convenient for learning policies, the rewards may not be readily available and crafting them may not be trivial.

Alternatively, the high-level control objectives can be formulated using linear temporal logic (LTL) \cite{baier2008}. LTL~provides a high-level intuitive language to formally specify~desired behaviors of a control task. For a given LTL specification, the objective can be described as finding a policy maximizing the probability that the specification is satisfied. Synthesis of such policies directly from LTL specifications has attracted significant interest (e.g.,~\cite{
kress2009, 
plaku2016
}).
Although LTL specifications can naturally express many characteristics of interest such as safety, sequencing, conditioning, persistence and liveness; optimization of a quantitative performance measure usually cannot be captured by an LTL specification.

Consequently, in this work, we focus on  control design %
problems where the objectives include satisfaction of desired properties expressed as LTL specifications as well as maximization of performance criteria represented by QoC rewards. In our problem setting, the LTL objectives are prioritized above the QoC objectives. In addition, we separate the safety specifications from the rest of the LTL specification as ensuring control safety is usually of the utmost~importance.

The control design problem in stochastic environments has been widely studied for multiple return objectives (e.g., \cite{roijers2013} and references therein), as well as 
for safety and LTL specifications \cite{svorevnova2016, hahn2019b, chatterjee2020, kalagarla2020}. %
Similarly, shielding-like approaches, which prioritize the safety objectives above the QoC objectives, have been introduced (e.g., \cite{
alshiekh2018,
avni2019
}).
These approaches require some partial knowledge such as a topology or an abstraction of the environment to determine unsafe actions to be eliminated or overridden. This allows for the use of reinforcement learning (RL) algorithms to learn an optimal policy subject to the safety constraints. Unfortunately, when the environment is completely unknown, these synthesis or shielding approaches cannot be directly used. %

Recently, there has been a considerable interest in the use of RL algorithms to learn policies for LTL objectives (e.g. \cite{
fu2014, %
icarte2018,
hahn2019a,
bozkurt2020
}). However, only few RL algorithms for multi-objectives with LTL specifications have been proposed for unknown environments. Studies~\cite{li2019b,aksaray2021}  considered time-bounded LTL specifications; \cite{kvretinsky2018}  proposed a model-based RL method that, under some assumptions about the environment, converges to a near optimal policy for an average QoC reward objective subject to the specifications;
\cite{hammond2021} introduced a model-free RL approach for multiple LTL~objectives,  maximizing the weighted sum of the satisfaction probabilities.

In this work, we introduce an RL algorithm for safety, LTL and QoC objectives with a lexicographic order where the safety objectives have the highest priority, while the LTL objectives have higher priority than the QoC-return objectives. 
As we demonstrate in a case study, 
it is crucial to consider the satisfaction of safety and LTL specifications as separate objectives with different priorities rather than the satisfaction of a single combined LTL specification because
such approach might result in a policy under which the probability of satisfying the safety specification is significantly reduced. 
We show that our algorithm converges to a policy maximizing the probability that the LTL specification is satisfied under the constraint that the probability of satisfying the safety specification is maximized. Furthermore, the derived policy is near-optimal in terms of QoC returns that can be obtained while having these maximum satisfaction probabilities. Our method is \textit{\textbf{completely model-free}} and does not rely on any assumptions about the environment and 
the specifications.

\section{Preliminaries and Problem Statement}
\label{sec:problem}

\subsection{Markov Decision Processes}
\vspace{-1pt}
We consider the stochastic environments modeled as Markov decision processes (MDPs).

\vspace{-2pt} 
\begin{definition} \label{def:mdps}
A (labeled) MDP is a tuple $\mathcal{M}{=}(S, s_0, A, P,\allowbreak  R, \gamma, \textnormal{AP}, L)$ where 
$S$ is a finite set of states;
$s_0$ is the initial state;
$A$ is a finite set of actions and $A(s)$ denotes the set of actions allowed in a state $s$;
$P: S \times A \times S \mapsto [0,1]$ is a probabilistic transition function such that $\sum_{s'\in S} P(s,a,s')=1$ if $a \in A(s)$, and $0$ otherwise;
$R: S \mapsto \mathbb{R}$ is a reward function;
$\gamma \in [0,1)$ is a discount factor;
\textnormal{AP} is a finite set of atomic propositions; and 
$L:S\mapsto 2^\textnormal{AP}$ is a labeling function.
\end{definition}

In a state $s$, the controller takes an action $a \in A(s)$ and makes a transition to a state $s'$ with probability (w.p.) $P(s,a,s')$, receiving a reward of $R(s')$ indicating the QoC, and a label $L(s')\subseteq AP$, a set of state properties.  The actions to be taken by the controller are determined by a \emph{policy}. In this study, we are interested in policies with limited memory.

\vspace{-2pt} 
\begin{definition} \label{def:policies}
A \textbf{finite-memory policy} for an MDP $\mathcal{M}$ is a tuple $\pi=(M,m_0,T,\mathfrak{A})$ where $M$ is a finite set of modes (memory states); $m_0$ is the initial mode; $T:M \times S \times M\mapsto [0,1]$ is a probabilistic transition function such that $T(m,s,m')$ is the probability that the policy switches to the mode $m'$ after visiting the state $s$ while operating in the mode $m$;
$\mathfrak{A}:M\times S \times A \mapsto [0,1]$ is a function that maps a given mode $m\in M$ and a state $s\in S$ to the probability that $a$ is taken in $s$ when the current mode is $m$. A finite-memory strategy is called \textbf{pure memoryless} if there is only one mode ($|M|=1$) and $\mathfrak{A}(m_0,s,a)$ is a point distribution assigning a probability of 1 to a single action in all states $s\in S$.
\end{definition}

Under a finite-memory policy, the action to be taken not only depends on the current state but also the current mode. After each transition, the policy may change its mode based on the visited state. The infinite sequence of states visited during an execution of the policy is called a \emph{path}, denoted by $\rho=s_0s_1\dots$, where $s_t$ is the state visited at time step $t$. For a given path $\rho$, we use $\rho[t]$ and $\rho[t{:}]$ to denote the state $s_t$ and the suffix $s_ts_{t+1}...$, respectively. The paths can be considered as sample sequences drawn from the \emph{Markov chain} (MC), denoted by $\mathcal{M}_\pi$, induced by the followed policy~$\pi$ in an MDP $\mathcal{M}$; the state space of the induced MC is composed of the $S$ and $M$, and the transitions are governed by $P$, $T$ and $\mathfrak{A}$.
We use $\mathcal{M}_\pi$ to denote the induced MC and $\rho \sim \mathcal{M}_\pi$ to denote a path randomly sampled from $\mathcal{M}_\pi$.

The QoC \emph{return} of a path $\rho$, denoted by $G(\rho)$, is the sum of discounted QoC rewards collected on the path~$\rho$ -- i.e.,
\begin{align}
    G(\rho) \coloneqq \sum\nolimits_{t=0}^\infty \gamma^tR(\rho[t]); \label{eq:return}
\end{align}
the QoC \emph{return objective} is to find a policy under which the expected QoC return of a path is maximized. When the reward and the probabilistic transition functions are unknown the optimal policies can be obtained using RL methods~\cite{sutton2018}.

\vspace{-1pt}
\subsection{Linear Temporal Logic}
\vspace{-1.2pt}
We adopt LTL to specify the desired properties of a control policy. LTL specifications can be formed according to the following grammar:
\begin{align}
    \hspace{-0.46em}\varphi \coloneqq  \mathrm{true} \mid a \mid \varphi_1 \wedge \varphi_2 \mid \neg \varphi \mid \bigcirc \varphi \mid \varphi_1 \textsf{U} \varphi_2, ~ {a\in\textnormal{AP}}.
\end{align} 

The semantics of an LTL formula $\varphi$ are defined over paths. A path $\rho$ either satisfies $\varphi$, denoted by $\rho \models \varphi$, or not ($\rho \not\models \varphi$). The satisfaction relation is recursively defined as follows: $\rho \models \varphi$ 
if $\varphi=a$ and $L(\rho[0])=a$;
if $\varphi=\varphi_1 \wedge \varphi_2$, $\rho \models \varphi_1$ and $\rho \models \varphi_2$;
if $\varphi=\neg \varphi'$ and $\rho \not\models \varphi'$;
if $\varphi=\bigcirc\varphi'$ (\emph{next} $\varphi'$) and $\rho[1{:}]\models\varphi'$;
if $\varphi=\varphi_1\textsf{U}\varphi_2$ ($\varphi_1$ \emph{until} $\varphi_2$) and there exists $t\geq0$ such that $\rho[t{:}]\models \varphi_2$ and for all $0\leq i<t$, $\rho[i{:}]\models \varphi_1$. 
The Boolean operators \emph{or} ($\vee$) and \emph{implies} ($\to$) can be easily obtained via: 
$\varphi_1 \vee \varphi_2 \coloneqq \neg(\neg\varphi_1 \wedge \neg\varphi_2)$, $\varphi_1\to \varphi_2 \coloneqq \neg \varphi_1 \vee \varphi_2$. 
In addition, we can derive the commonly used temporal operators \emph{eventually} ($\lozenge$) and \emph{always} ($\square$) using: 
$\lozenge \varphi \coloneqq \mathrm{true} \ \textsf{U}\ \varphi$, and 
$\square \varphi \coloneqq \neg (\lozenge \neg \varphi)$.

For an LTL formula, a limit-deterministic Büchi automaton (LDBA) that only accepts the paths satisfying the formula can be automatically constructed \cite{sickert2016}.

\vspace{-2pt} 
\begin{definition} \label{def:ldba}
An LDBA of an LTL formula $\varphi$ is a tuple $\mathcal{A}=(Q,  q_0, \Sigma, \delta, B)$ where $Q$ is a finite set of automata states, $q_0 \in Q$ is an initial state, $\Sigma$ is a finite alphabet, $\delta: Q \times (\Sigma \cup \{\varepsilon\}) \mapsto 2^Q$ is a transition function, and $B\subseteq Q$ is a set of \emph{accepting states}. The set $\delta(q, \theta)$ is singleton (i.e., $|\delta(q, \theta)| = 1$)  for all $q \in Q$ and  $\theta \in \Sigma$. The set $Q$ can be partitioned into an initial and an accepting component, namely $Q_I$ and $Q_A$, such that all the accepting states are in $Q_A$ (i.e., $B \subseteq Q_A$), and $Q_A$ does not have any $\varepsilon$ or outgoing transitions; i.e., for any $q\in Q_A$,
\begin{equation}
    \delta(q, \theta) \subseteq \begin{cases}
        \varnothing & \textnormal{ if } \theta=\varepsilon, \\
        Q_A & \textnormal{ if } \theta \in \Sigma.
    \end{cases}
\end{equation}
A path $\rho$ in an MDP $\mathcal{M}$ is accepted if there is an execution $\sigma=q_0q_1\dots$ such that $q_{t+1} \in \delta(q_t, L(\rho[t])) \cup \delta(q_t,\varepsilon)$ for all $t\geq0$ and $\textnormal{Inf}(\sigma) \cap B \neq \varnothing$, where $\textnormal{Inf}(\sigma)$ denotes the set of automata states visited by $\sigma$ infinitely many times.
\end{definition}

We write $\mathcal{A}_\varphi$ to denote an LDBA constructed by an LTL formula $\varphi$. All the accepting states of an LDBA $\mathcal{A}_\varphi$ are in the accepting component and $\varphi$ is satisfied by a path if and only if the labels of states on the path trigger an execution that visits some of these accepting states infinitely often.
The only nondeterministic transitions in an LDBA are the outgoing $\epsilon$-transitions from the states in the initial component $Q_I$.
When an LDBA consumes the label of a state $L(s)$, it deterministically transitions from its current state $q$ to the state in the singleton set $\delta(q,L(s))$ if $\delta(q,\varepsilon)$ is empty. Once the accepting component $Q_A$ is reached, it is not possible to make a nondeterministic transition.
LDBAs can be constructed in a way that any $\epsilon$-transition leads to a state in the accepting component, allowing for at most one $\epsilon$-transition in an execution. 
These LDBAs are called \emph{suitable}, as they facilitate quantitative model-checking of MDPs~\cite{sickert2016}, and we henceforth assume that any given LDBA is suitable.

\textbf{\emph{Safety LTL}} is an important fragment of LTL, which can be used to specify the properties ensuring ``something bad never happens''. For instance, any LTL formula in a positive normal form without any temporal modalities other than $\bigcirc$ and $\square$ is a safety formula \cite{kupferman2001}. Safety formulas can be translated into simpler automata where the transitions are deterministic and all non-accepting automata states are absorbing, meaning visiting a non-accepting automata state results in rejection. We refer these automata as \emph{safety automata} \cite{kupferman2001}.

\subsection{Problem Formulation}
In this work, we focus on %
RL for three lexicographically ordered objectives: \emph{safety} (primary), \emph{LTL} (secondary) and \emph{QoC} (tertiary). In our problem setting, the safety constraints are specified as an LTL safety formula and meeting these constraints is of highest priority. The other desired temporal properties are expressed as a general LTL formula and we are interested in policies satisfying the LTL formula with the highest probability. Finally, among such policies we want to learn the ones maximizing the expected QoC return. Hence, we formally state our problem as follows.

\begin{problem} \label{problem}
For a given MDP $\mathcal{M}$ where the transition probabilities $P$ and the rewards $R$ are unknown, a safety specification $\psi$ and an LTL specification $\varphi$, design a model-free RL algorithm that learns a policy $\pi_{\psi,\varphi}^R \in \Pi_{\psi,\varphi}^R$ where
\begin{align}
    \Pi_{\psi,\varphi}^R &\coloneqq \argmax_{\pi \in \Pi_{\psi,\varphi}} \mathbb{E}_{\rho \sim \mathcal{M}_{\pi}}\left[G(\rho)\right], \label{eq:safety_ltl_optimal_policies} \\
    \Pi_{\psi,\varphi} &\coloneqq \argmax_{\pi \in \Pi_\psi} Pr_{\rho \sim \mathcal{M}_{\pi}}\left\{\rho \mid \rho \models \varphi \right\}, \label{eq:safety_ltl_policies} \\
    \Pi_\psi &\coloneqq \argmax_{\pi \in \Pi} Pr_{\rho \sim \mathcal{M}_{\pi}}\left\{\rho \mid \rho \models \psi \right\}, \label{eq:safety_policies}
\end{align}
and $\Pi$ denotes the set of all finite-memory policies.
\end{problem}

Intuitively, our objective is to learn a \emph{lexicographically optimal} policy $\pi_{\psi,\varphi}^R$ that first maximizes the probability of ensuring the safety, then the probability of satisfying the LTL specification and lastly, the sum of discounted QoC rewards. However, such optimal policies might not always exist as illustrated in the following example.

\vspace{-2pt} \begin{example}\label{example:no_optimal}
Consider the MDP from Fig.~\ref{fig:no_optimal} where the LTL specification $\varphi$ is $\square\lozenge b$ and there is no safety specification (i.e., $\psi=\mathrm{true}$). The policy $\pi^*$ choosing $\beta_2$ in $s_0$ maximizes the QoC return. Now, let $\pi_{\psi,\varphi}^R$ be a policy in the set $\Pi_{\psi,\varphi}^R$ defined in \eqref{eq:safety_ltl_optimal_policies}. The expected value of the QoC return obtained under $\pi_{\psi,\varphi}^R$ is strictly less than the one obtained under $\pi^*$ because $\beta_1$ needs to be eventually taken by $\pi_\varphi^*$ to satisfy $\varphi$, which leads to a reward of zero. Thus, a policy that randomly follows either  $\pi^*$ or $\pi_{\psi,\varphi}^R$ in $s_0$ increases the expected QoC return; the more often $\beta_2$ is chosen, the higher QoC return is obtained. In addition, such a mixed policy satisfies $\varphi$ as well, because the mixed policy takes $\beta_1$ in $s_1$ with some positive probability due to $\pi_{\psi,\varphi}^R$. As a result, we obtain a policy better than $\pi_{\psi,\varphi}^R$, which leads to a contradiction.
\end{example}

\begin{figure}
    \centering
\scalebox{0.76}{
\begin{tikzpicture}[
state/.style={circle, draw=black, very thick, minimum size=30pt, inner sep=1pt},
action/.style={circle, inner sep=0pt}]
\node[state]  (s0) {\shortstack{{\Large $s_0$}\\ \{\},+1}};
\node[action] (b1) [right=of s0] {$\beta_1$};
\node[action] (b2) at (1,-1) {$\beta_2$};
\node[state]  (s1) [right=of b1] {\shortstack{{\Large $s_1$}\\ \{b\},0}};
\node[action] (b3) at (5,1) {$\beta_3$};
\draw[->, very thick] (-1,0) -- (s0.west);
\draw[very thick] (s0.east) -- (b1.west);
\draw[->, very thick] (b1.east) -- (s1.west);
\draw[very thick] (s0.east) .. controls +(right:10pt) and +(up:10pt).. (b2.north);
\draw[->, very thick] (b2.west) .. controls +(left:10pt) and +(down:10pt).. (s0.south);
\draw[very thick] (s1.east)  .. controls +(right:10pt) and +(down:20pt).. (b3.south);
\draw[->, very thick] (b3.west) .. controls +(left:10pt) and +(up:20pt).. (s0.north);
\end{tikzpicture}
}
\vspace{-5pt}
    \caption{An example deterministic MDP. The circles represent the states and the arrows represent the actions.  The labels and QoC rewards of the states $s_0$ and $s_1$ are $\{\},{+}1$ and $\{b\},0$ respectively.}
    \label{fig:no_optimal}
    \vspace{-15pt}
\end{figure}
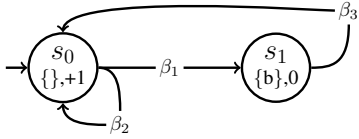

Consequently, we consider \emph{near-optimality} for the QoC objective. Specifically, for a given $\nu{>}0$, a policy $\pi_{\psi,\varphi}^{R\nu}$ is called $\nu$-optimal if it belongs to the set $\Pi_{\psi,\varphi}^{R\nu}$ defined as
\vspace{5pt}
\begin{align}
    \Pi_{\psi,\varphi}^{R\nu} \coloneqq \big\{\pi \in \Pi_{\psi,\varphi}  \ \big| \  v_{\psi,\varphi}^R - \mathbb{E}_{\rho \sim \mathcal{M}_{\pi}}\left[G(\rho)\right] \leq \nu \big\}, \label{eq:safety_ltl_near_optimal_policies}
\end{align}

\vspace{5pt}
\noindent where $v_{\psi,\varphi}^R \coloneqq \sup_{\pi \in \Pi_{\psi,\varphi}} \mathbb{E}_{\rho \sim \mathcal{M}_{\pi}}\left[G(\rho)\right]$. We note that the supremum exists as we assume the rewards are bounded.

\section{Model-free RL for Lexicographic Safety, LTL and QoC objectives}
\label{sec:mfRL}

In this section, we introduce our model-free RL algorithm that learns lexicographically near-optimal policies defined in \eqref{eq:safety_ltl_near_optimal_policies}. Our algorithm first constructs a product MDP by composing the initial MDP with the automata of the given safety and LTL specifications, where the satisfaction of these specifications is reduced into meeting the acceptance conditions of the automata; it then uses the safety and the LTL rewards crafted for the acceptance conditions as described in~\cite{bozkurt2020} and the QoC rewards of the original MDP to learn the policies via a model-free approach similar to~\cite{gabor1998}.

\subsection{Product MDP Construction}
The first step in our approach is to construct a product MDP of the given MDP and the automata derived from the given safety and LTL specifications. The product MDP is merely a representation of the synchronous execution of the MDP and the automata. Thus, even if the transition graph and probabilities of the MDPs are unknown, the product MDP can be conceptually constructed; i.e., the resulting probabilities, and the transition graph, will be unknown.

\begin{definition}
A product MDP of an MDP $\mathcal{M}$, a safety automaton $\mathcal{A}_\psi$, and an LDBA $\mathcal{A}_\varphi$ is a tuple $\mathcal{M}^\times=(S^\times,s_0^\times, A^\times, P^\times, R^\times, \gamma, B_\psi^\times, B_\varphi^\times)$ such that 
$S^\times = S \times Q_\psi \times Q_\varphi$ is the set of product states;
$s_0^\times = \langle s_0, q_{0\psi}, q_{0\varphi}\rangle$ is the initial product state;
$A^\times = A \cup \{\epsilon_{q_\varphi}\mid q_\varphi\in Q_\varphi\}$ where $A^\times(\langle s, q_\psi, q_\varphi \rangle)$ denotes the set $A(s) \cup  \{\epsilon_{q_\varphi'}\mid q_\varphi'\in \delta_\varphi(q_\varphi,\varepsilon)\}$;
$P^\times:S^\times\times A^\times \times S^\times \mapsto [0,1]$ is the transition function such that $P^\times(\langle s, q_\psi, q_\varphi \rangle,a,\langle s', q_\psi', q_\varphi'\rangle)=$
\begin{align}
\begin{cases}
    P(s,a,s') & \textnormal{if } a {\in} A(s),  q_\psi' {\in} \delta_\psi(q_\psi,L(s)), q_\varphi' {\in} \delta_\varphi(q_\varphi,L(s)), \\
    1 & \textnormal{if } a {=} \epsilon_{q_\varphi'}, q_\varphi' {\in}\delta_\varphi(q_\varphi,\varepsilon), q_\psi'{=}q_\psi, s{=}s', \\
    0 & \textnormal{otherwise;}
\end{cases}\notag %
\end{align}
$R^\times: S^\times \mapsto \mathbb{R}$ is the reward function such that $R^\times(\langle s, q_\psi, q_\varphi \rangle)=R(s)$;
$\gamma$ is the discount factor;
$B_\psi^\times = S \times B_\psi \times Q_\varphi$ is the set of accepting states for $\psi$;
$B_\varphi^\times = S \times Q_\psi \times B_\varphi$ is the set of accepting states for $\varphi$.
\end{definition}

Due to the $\varepsilon$-transitions in $\mathcal{A}_\varphi$, $\mathcal{M}^\times$ has $\varepsilon$-actions in addition to the actions in $\mathcal{M}$. Taking an $\varepsilon$-action $\varepsilon_{q_\varphi}$ does not cause an actual transition in $\mathcal{M}$; it only makes $\mathcal{A}_\varphi$ to move to $q_\varphi$. A policy $\pi^\times$ for $\mathcal{M}^\times$ induces a finite-memory policy $\pi$ for $\mathcal{M}$ where the automata states and transitions act as the memory mechanism. The mode of the induced policy $\pi$ at any time step is represented by the joint state $\langle q_\psi, q_\varphi \rangle$ of the states that $\mathcal{A}_\psi$ and $\mathcal{A}_\varphi$ are in. When a state~$s$ is visited, $\pi$ switches its mode to $\langle q_\psi', q_\varphi' \rangle$, where $q_\psi'$ and $q_\varphi'$ are the automata states that $\mathcal{A}_\psi$ and $\mathcal{A}_\varphi$ transition to after consuming $L(s)$. The induced policy $\pi$ then takes the action that $\pi^\times$ takes in $\langle s, \langle q_\psi', q_\varphi' \rangle \rangle$. %

A path $\rho^\times$ in $\mathcal{M}^\times$ satisfies the \emph{safety condition}, denoted by $\rho^\times \models \square B_\psi^\times $, if $\rho^\times$ only visits the product states in $B_\psi^\times$. This means that its induced path $\rho$ in $\mathcal{M}$ is accepted by $\mathcal{A}_\psi$ and thereby satisfying the safety specification $\psi$. Similarly, if $\rho^\times$ visits some product states in $B_\varphi^\times$ infinitely many times, $\rho^\times$ satisfies the \emph{B\"uchi condition}, denoted by $\rho^\times \models \square \lozenge B_\varphi^\times$, and its induced path $\rho$ satisfies the LTL specification $\varphi$. 
Finally, the QoC return of $\rho^\times$, denoted by $G^\times(\rho^\times)$, is equivalent to the QoC return of its induced path. 

\subsection{Lexicographically Optimal Policies for Product MDPs}

Pure memoryless policies suffice for safety, B\"uchi\cite{baier2008}, and QoC objectives\cite{sutton2018}. However, mixing might be necessary to obtain a lexicographically near-optimal policy defined in~\eqref{eq:safety_ltl_near_optimal_policies}, as illustrated in Example~\ref{example:no_optimal}.
Hereafter, we focus on memoryless, but possibly mixed, policies for the product MDP as their induced policies in the original MDP suffice for lexicographic near-optimality as stated in the lemma.

\vspace{-2pt} 
\begin{lemma} \label{lemma:policies_in_product_mdp}
For a given MDP $M$, a safety automaton $\mathcal{A}_\psi$, an LDBA $\mathcal{A}_\varphi$ and their product MDP  $\mathcal{M}^\times$, let $\Pi_{\psi,\varphi}^{R\nu\times}$, $\Pi_{\psi,\varphi}^\times$ and $\Pi_\psi^\times$ be the sets of policies defined as follows:
\begin{align}
    \hspace{-1px}\Pi_{\psi,\varphi}^{R\nu\times} &{\coloneqq}  \left\{\pi^\times\hspace{-1px} {\in} \Pi_{\psi,\varphi}^\times \big|  v_{\psi,\varphi}^{R\times} {-} \mathbb{E}_{\rho^\times\hspace{-1px} \sim  \mathcal{M}^\times_{\pi^\times}}\hspace{-2px}\left[G^\times\hspace{-1px}(\rho^\times\hspace{-1px})\right]\hspace{-1px} \leq \nu \right\}\hspace{-1px}, \label{eq:safety_ltl_near_optimal_product_policies}
\end{align}
\begin{align}
    \Pi_{\psi,\varphi}^\times &{\coloneqq} \argmax_{\pi^\times \in \Pi_\psi^\times } Pr_{\rho^\times  \sim \mathcal{M}^\times_{\pi^\times}}\big\{\rho^\times \mid \rho^\times {\models} \square \lozenge B_\varphi^\times \big\}, \label{eq:safety_ltl_product_policies} 
\end{align}
\begin{align}
    \Pi_\psi^\times &{\coloneqq} \argmax_{\pi^\times  \in \Pi^\times } Pr_{\rho^\times  \sim \mathcal{M}^\times_{\pi^\times}}\big\{\rho^\times \mid \rho^\times {\models} \square B_\psi^\times \big\}; \label{eq:safety_product_policies}
\end{align}
here, $\Pi^\times$ denotes the set of all memoryless policies for $\mathcal{M}^\times$ and $v_{\psi,\varphi}^{R\times} \coloneqq \sup_{\pi^\times\hspace{-1px} {\in} \Pi_{\psi,\varphi}^\times} \mathbb{E}_{\rho^\times\hspace{-1px} \sim  \mathcal{M}^\times_{\pi^\times}}\hspace{-2px}\left[G^\times\hspace{-1px}(\rho^\times\hspace{-1px})\right]$. Then $\Pi_{\psi,\varphi}^{R\nu\times}$ is non-empty and any policy in $\Pi_{\psi,\varphi}^{R\nu\times}$ induces a policy for $\mathcal{M}$ belonging to the set $\Pi_{\psi,\varphi}^{R\nu}$ defined in \eqref{eq:safety_ltl_near_optimal_policies}.
\end{lemma}
Before proving this lemma, we first introduce notation for satisfaction probabilities. We write $Pr_{\pi^\times}(s^\times \models \phi)$ for the probability of satisfying $\phi$ after visiting $s^\times$; i.e.,
\vspace{4pt}
\begin{align}
    Pr_{\pi^\times}(s^\times \models \phi) \coloneqq Pr_{\rho_{s^\times}^\times \sim \mathcal{M}_{\pi^\times,s^\times}^\times}\big\{\rho_{s^\times}^\times \mid \rho_{s^\times}^\times \models \phi \big\},
\end{align}
where $\rho_{s^\times}^\times \sim \mathcal{M}_{\pi^\times,s}^\times$ denotes the random path drawn from the $\mathcal{M}_{\pi^\times,s^\times}^\times$, which is obtained from $\mathcal{M}_{\pi^\times}^\times$ by changing its initial state to $s^\times$. Similarly, we denote the satisfaction probability after taking $a^\times$ in $s^\times$ by $Pr_{\pi^\times}\left((s^\times,a^\times) \models \phi\right)$, which can be described as
\vspace{4pt}
\begin{align}
   \hspace{-5px}Pr_{\pi^\times}\hspace{-3px}\left((s^\times\hspace{-3px},a^\times\hspace{-1px}) {\models} \phi\right) \hspace{-2px} {\coloneqq} \hspace{-9px} \sum_{s^\times{'} \in S^\times} \hspace{-7px}P^\times\hspace{-2px}(s^\times\hspace{-3px},a^\times\hspace{-3px},s^\times{'})Pr_{\pi^\times}\hspace{-1px}(s^\times{'} {\models} \phi), \hspace{-4px}
\end{align}
and we denote the maximum satisfaction probabilities by
\vspace{4pt}
\begin{align}
    Pr_{\max}\hspace{-2px}\left((s^\times\hspace{-2px},a^\times) \models \phi\right) \coloneqq \max_{\pi^\times} Pr_{\pi^\times}\hspace{-2px}\left((s^\times\hspace{-2px},a^\times) \models \phi\right).
\end{align}

\vspace{-4pt}
\begin{proof}
Let $A^\times_\psi(s^\times)$ denote the action set maximizing satisfaction probabilities for the safety condition in $s^\times$ defined~as
\vspace{-8pt}
\begin{align}
    A^\times_\psi(s^\times) \coloneqq \argmax_{a^\times \in A^\times(s^\times)} Pr_{\max}\big((s^\times,a^\times)\models \square B^\times_\psi\big). \label{eq:safe_action_set}
\end{align}
A policy $\pi^\times$ maximizes the probability of satisfying the safety condition (i.e., $\pi^\times {\in} \Pi_\psi^\times$) if and only if it does not choose an action $a^\times {\not\in} A^\times_\psi(s^\times)$ (this can be shown using the techniques provided for the reachability probabilities in \cite{baier2008}); thus, all the actions not in $A^\times_\psi(s^\times)$ can be pruned.
We now can define similar action sets for the B\"uchi condition as
\vspace{4pt}
\begin{align}
    A^\times_{\psi,\varphi}(s^\times) \coloneqq \hspace{-1pt} \argmax_{a^\times \in A^\times_\psi(s^\times)} \hspace{-1pt} Pr_{\max}\big((s^\times,a^\times)\models \square \lozenge B^\times_\varphi\big), \label{eq:ltl_action_set}
\end{align}
where the maximum probabilities are obtained for the pruned MDP. Choosing from $A^\times_{\psi,\varphi}(s^\times)$ in state $s^\times$, similar to~the safety condition, is necessary but not sufficient for the B\"uchi condition; the actions should eventually lead to a state in $B^\times_\varphi$ whenever it is possible. This can be overcome by following a policy $\pi^\times$ choosing any action in $A^\times_{\psi,\varphi}(s^\times)$~with some positive probability. Under $\pi^\times$, a state in $B^\times_\varphi$ is eventually reached and visited infinitely many times
(i.e., $\pi^\times {\in} \Pi_{\psi,\varphi}^\times$).

We can now focus on the MDP obtained by pruning all the actions that do not belong $A^\times_{\psi,\varphi}(s^\times)$ in any state $s^\times$. Let $\pi^\times_{u}$ denote a policy that uniformly chooses a random action and $\pi^\times_{o}$ be an optimal policy maximizing the expected QoC return in this pruned MDP. As previously discussed, 
a policy~$\pi^\times_{\upsilon}$ that follows $\pi^\times_{u}$ w.p. $\upsilon$ and $\pi^\times_{o}$ w.p. $1{-}\upsilon$ belongs to $\Pi_{\psi,\varphi}^\times$.
As $\upsilon$ goes to zero, the expected QoC return under $\pi^\times_{\upsilon}$ approaches the maximum expected QoC return that can be obtained in the pruned MDP. Thus, there exists a sufficiently small $\upsilon$ such that $\pi^\times_{\upsilon}$ belongs to $\Pi_{\psi,\varphi}^{R\nu\times}$.
\end{proof}
\vspace{-12pt}

\subsection{Reduction from Safety and B\"uchi Conditions to Rewards}
We provide a reduction from the safety and B\"uchi acceptance conditions to the rewards introduced in~\cite{bozkurt2020} to enable model-free RL to learn lexicographically optimal policies. We note that these crafted reward are independent from the QoC rewards that are native to the original MDP.
The idea behind the reduction is to provide small rewards whenever an acceptance state is visited in order to encourage the repeated visits, and discount %
in a way that ensures the values approach the satisfaction probabilities as the rewards provided goes to zero. Since the reduction requires state-based discounting, we extend the definition of the return in this context as
\vspace{2pt}
\begin{align}
    G_{R^\times,\Gamma^\times}^\times(\rho^\times) \coloneqq \sum_{t=0}^\infty \left( \prod_{i=0}^t \Gamma^\times(\rho^\times[i]) \right) R^\times(\rho^\times[t]), \label{eq:return_updated}
\end{align}
where $R^\times$ and $\Gamma^\times$ are given reward and discount functions.

We further define the near-optimal action set in a state $s^\times$ for given $R^\times$, $\Gamma^\times$ and $\tau>0$ as
\vspace{4pt}
\begin{align}
    \hspace{-100pt}A^\times_{R^\times,\Gamma^\times,\tau}(s^\times) \coloneqq \big\{ a^\times {\in} A^\times(s^\times) \mid
\end{align}
\vspace{-14pt}
\begin{align}
    &\argmax_{\mathfrak{a}^\times{\in} A^\times(s^\times)}q^\times_{\max,R^\times,\Gamma^\times}(s^\times,\mathfrak{a}^\times) {-} q^\times_{\max,R^\times,\Gamma^\times}(s^\times,a^\times) \leq \tau  \big\} \notag
\end{align}

\vspace{2pt}\noindent
where $q^\times_{\max,R^\times,\Gamma^\times}(s^\times,a^\times)$ denotes the maximum expected return for $R^\times$ and $\Gamma^\times$ obtained after taking $a^\times$ in $s^\times$.

The following lemma provides a way to obtain the optimal action sets for the safety condition via learning the near-optimal actions for the reward and discount functions crafted for the condition.

\vspace{-2pt} 
\begin{lemma}\label{lemma:reduction_from_safety}
There exist sufficiently small $r_\psi>0$ and $\tau_\psi>0$ such that for all $s^\times \in S^\times$,
\vspace{3pt}
\begin{align}
    A^\times_\psi(s^\times\hspace{-1pt}) &= \hspace{-1pt}\big\{ a^\times\hspace{-1pt} {\in} A^\times\hspace{-1pt}(s^\times\hspace{-1pt}) \ \big| \  v^\times_\psi\hspace{-1pt}(s^\times\hspace{-1pt}) {-} q^\times_\psi\hspace{-1pt}(s^\times\hspace{-3pt},a^\times\hspace{-1pt}) \leq \tau_\psi  \big\}; \notag
\end{align}
here $A^\times_\psi(s^\times)$ is defined in~\eqref{eq:safe_action_set}, $q^\times_\psi\hspace{-1pt}(s^\times\hspace{-3pt},a^\times\hspace{-1pt})$ denotes the maximum expected return that can be obtained by taking $a^\times$ in $s^\times$ for $\Gamma_\psi^\times(s^\times) := 1-r_\psi$ and 
\vspace{-2pt}
\begin{align}
    R_\psi^\times(s^\times) := \begin{cases}
        r_\psi &  \hspace{-2pt}\textnormal{if }s^\times {\in} B^\times_\psi, \\
        0 & \hspace{-2pt}\textnormal{if }s^\times {\not\in} B^\times_\psi;
    \end{cases} \notag %
\end{align}
finally, $v^\times_\psi\hspace{-1pt}(s^\times\hspace{-1pt}) \coloneqq \max_{\mathfrak{a}^\times\hspace{-1pt}{\in} A^\times\hspace{-1pt}(s^\times\hspace{-1pt})}\hspace{-1pt}q^\times_\psi\hspace{-1pt}(s^\times\hspace{-3pt},\mathfrak{a}^\times\hspace{-1pt})$.
\end{lemma}

\vspace{-6pt}\begin{proof}
Any safety condition can be considered as a B\"uchi condition and therefore the proof immediately follows from Theorem 1 in~\cite{bozkurt2020}. The idea is that as $r_\psi$ goes to zero, the safety return of a path approaches one if it stays in safe states and zero if it reaches an unsafe state; hence, the maximum expected safety return approaches to the maximum satisfaction probability. Since the MDP is finite, for sufficiently small $\tau_\psi>0$ there must exist $r_\psi>0$ such that $A^\times_\psi(s^\times)$ and the near-optimal action sets are equivalent.
\end{proof}
\vspace{-4pt}

We now establish a similar result for the B\"uchi condition.

\begin{lemma}\label{lemma:reduction_from_ltl}
There exist sufficiently small $r_\varphi{>}0$ and $\tau_\varphi{>}0$ such that for all $s^\times {\in} S^\times$,
$
    A^\times_{\psi,\varphi}(s^\times\hspace{-1pt}) {=} \hspace{-1pt}\{ a^\times\hspace{-1pt} {\in} A^\times_\psi\hspace{-1pt}(s^\times\hspace{-1pt}) {\mid}\allowbreak  v^\times_{\psi,\varphi}\hspace{-1pt}(s^\times\hspace{-1pt}) {-} q^\times_{\psi,\varphi}\hspace{-1pt}(s^\times\hspace{-3pt},a^\times\hspace{-1pt}) {\leq} \tau_\varphi \}
$;
here $A^\times_{\psi,\varphi}(s^\times)$ is from \eqref{eq:ltl_action_set};  $q^\times_{\psi,\varphi}\hspace{-1pt}(s^\times\hspace{-3pt},a^\times\hspace{-1pt})$ denotes the maximum expected return that can be obtained by taking $a^\times$ in $s^\times$ for

\begin{align}
    R_\varphi^\times(s^\times) {:=} \begin{cases}
        r_\varphi &  \hspace{-3pt}\textnormal{if }s^\times {\in} B^\times_\varphi, \\
        0 & \hspace{-3pt}\textnormal{if }s^\times {\not\in} B^\times_\varphi,
    \end{cases} \
    \Gamma_\varphi^\times(s^\times) {:=} \begin{cases}
        1{-}r_\varphi &  \hspace{-3pt}\textnormal{if }s^\times {\in} B^\times_\varphi, \\
        1{-}r_\varphi^2 & \hspace{-3pt}\textnormal{if }s^\times {\not\in} B^\times_\varphi;
    \end{cases} \notag 
\end{align}
in the MDP where all the actions not in $A^\times_\psi(s^\times{'})$ are pruned in each $s^\times{'}$; and $v^\times_{\psi,\varphi}\hspace{-1pt}(s^\times\hspace{-1pt}) \coloneqq \max_{\mathfrak{a}^\times\hspace{-1pt}{\in} A^\times\hspace{-1pt}(s^\times\hspace{-1pt})}\hspace{-1pt}q^\times_{\psi,\varphi}\hspace{-1pt}(s^\times\hspace{-3pt},\mathfrak{a}^\times\hspace{-1pt})$.
\end{lemma}

\vspace{-4pt}\begin{proof}
The proof follows from Theorem 1 in \cite{bozkurt2020}. Here, we provide the key idea. Unlike the safety condition, the B\"uchi condition requires repeatedly visiting some states in $B^\times_\varphi$ and the frequency of the visits is not important. To capture that, the rewards are discounted less in the states that do not belong to $B^\times_\varphi$; as $r_\varphi$ goes to zero, the discounting due to visiting these states vanishes. In the limit, the B\"uchi return of a path is $1$ only if the path visits some states in $B^\times_\varphi$ infinitely many times, and $0$ otherwise.
\end{proof}
\vspace{-12pt}

\subsection{Model-Free Learning Algorithm}

We now unify these ideas into a single RL algorithm. For any state $s^\times$, our algorithm simultaneously learns the optimal action set
for the safety condition $A^\times_\psi(s^\times)$, and among those actions learns the ones optimal for the B\"uchi condition $A^\times_{\psi,\varphi}(s^\times)$. Additionally, it learns an optimal $\upsilon$-greedy policy that chooses the best action among $A^\times_{\psi,\varphi}(s^\times)$ in $s^\times$ for the QoC objective w.p. $1{-}\upsilon$ and uniformly chooses a random action in $A^\times_{\psi,\varphi}(s^\times)$ w.p. $\upsilon$. 

Our algorithm uses Q-learning to obtain the sets $A^\times_\psi(s^\times)$ and $A^\times_{\psi,\varphi}(s^\times)$ for each state $s^\times$, and SARSA to obtain a near-optimal policy belonging to the set $\Pi_{\psi,\varphi}^{R\nu\times}$. Both Q-learning and SARSA learn an estimate of the value, which is the expected QoC return obtained after taking $a^\times$ in $s^\times$, denoted by $Q(s^\times,a^\times)$ \cite{sutton2018}. The update of these estimates in our approach is shown in Algorithm~\ref{alg:update_q}. After observing a transition $(s^\times, a^\times, r, s^\times{'}, a^\times{'})$, SARSA updates $Q_{\psi,\varphi}^R(s^\times,a^\times)$ using $r + \gamma Q_{\psi,\varphi}^R(s^\times{'},a^\times{'})$ as the target value. Q-learning, on the other hand, updates $Q_\psi(s^\times,a^\times)$ and $Q_{\psi,\varphi}(s^\times,a^\times)$ using $r$ and the maximum value estimate in the next state $s^\times{'}$, independently from the action taken. 

Under regular conditions on the step size $\alpha$ and a proper exploration policy visiting every state action pair %
large number of  
times, Q-learning converges to the optimal values and SARSA converges to the values of the policy being followed \cite{sutton2018}. In our approach, an $\epsilon$-greedy policy where $\epsilon$ gradually decreased to zero is used as the exploration policy. The greedy action to be taken is chosen from the estimate $\hat{A}^\times_{\psi,\varphi}(s^\times)$, which is obtained using $Q_\psi(s^\times,a^\times)$ and $Q_{\psi,\varphi}(s^\times,a^\times)$ as described in Algorithm~\ref{alg:choose_action}. Lastly, an $\upsilon$-greedy policy decides if a random action or a greedy one for the QoC objective will taken from $\hat{A}^\times_{\psi,\varphi}(s^\times)$.

We now show that our overall model-free RL approach summarized in Algorithm~\ref{alg:main} learns a lexicographically near-optimal policy assuming that the step size $\alpha$ and the exploration parameter $\epsilon$ are properly decreasing to 0. 

\begin{theorem}
Consider an MDP $\mathcal{M}$, a safety specification $\psi$, an LTL specification $\varphi$, and a near-optimality parameter $\nu$. Algorithm~\ref{alg:main} converges to a policy $\pi^\times$ for the product MDP $\mathcal{M}^\times$, inducing a policy $\pi$ for $\mathcal{M}$ belonging to the set $\Pi_{\psi,\varphi}^{R\nu}$ from \eqref{eq:safety_ltl_near_optimal_policies}, for sufficiently small positive $r_\psi,r_\varphi,\tau_\psi,\tau_\varphi$ and~$\upsilon$.
\end{theorem}

\begin{proof}
Due to Lemma~\ref{lemma:reduction_from_safety}, Lemma \ref{lemma:reduction_from_ltl} and the convergence of Q-learning, the estimate $\hat{A}^\times_\psi(s^\times)$ and thereby $\hat{A}^\times_{\psi,\varphi}(s^\times)$ converge to $A^\times_\psi(s^\times)$ and $A^\times_{\psi,\varphi}(s^\times)$, respectively for sufficiently small positive $r_\psi,r_\varphi,\tau_\psi,\tau_\varphi$. Similarly, due to Lemma~\ref{lemma:policies_in_product_mdp}, there exists a $\upsilon>0$ ensuring the near-optimality.
\end{proof}
\vspace{-5pt}

\begin{figure*}
\begin{minipage}{0.3\textwidth}
\begin{algorithm}[H]
\caption{Model-free RL for lexicographical safety, LTL and QoC objectives}
\label{alg:main}
\begin{algorithmic}
\itemindent=-10pt
\begin{footnotesize}
\STATE \textbf{Input:} Safety $\psi$, LTL $\varphi$, MDP $\mathcal{M}$
\STATE Translate $\psi$ and $\varphi$ to $\mathcal{A}_\psi$ and $\mathcal{A}_\varphi$
\STATE Construct $\mathcal{M}^\times$ of $\mathcal{M}$, $\mathcal{A}_\psi$ and $\mathcal{A}_\varphi$
\STATE Initialize $Q_\psi$, $Q_{\psi,\varphi}$ and $Q_{\psi,\varphi}^R$ on $\mathcal{M}^\times$
\FOR {$i=0,1,\dots$}
\FOR {$t=0$ {\bfseries to} $\mathcal{T}-1$}
\STATE $a_t^\times {\leftarrow} \textnormal{get\_action}$($Q_\psi, Q_{\psi,\varphi}, Q_{\psi,\varphi}^R$, $s_t^\times$)
\STATE Take $a_t^\times$ and observe $r_t$ and $s^\times_{t+1}$
\IF{$t>0$}
\STATE \textnormal{update\_Qs}($Q_\psi, Q_{\psi,\varphi}, Q_{\psi,\varphi}^R,$
\STATE \hspace{20pt}$(s^\times_{t-1}, a^\times_{t-1}, r_{t-1}, s^\times_t, a^\times_t)$)\vspace{-4pt}
\ENDIF \vspace{-1pt}
\ENDFOR \vspace{-4pt}
\ENDFOR
\end{footnotesize}
\end{algorithmic}
\end{algorithm}
\end{minipage}
\hfill
\begin{minipage}{0.32\textwidth}
\begin{algorithm}[H]
\caption{update\_Qs}
\label{alg:update_q}
\begin{algorithmic}
\itemindent=-10pt
\begin{footnotesize}
\vspace{-1pt}
\STATE \textbf{Input}: $Q_\psi, Q_{\psi,\varphi}, Q_{\psi,\varphi}^R,(s^\times, a^\times, r, s^\times{'}, a^\times{'})$\vspace{-1pt}
\IF{$s^\times{'} {\in} B_\psi^\times$}\vspace{-1pt}
\STATE $Q_\psi(s^\times,a^\times) \leftarrow  (1{-}\alpha)Q_\psi(s^\times,a^\times)$\vspace{-1pt}
\STATE \hspace{2pt}${+} \alpha (r_\psi {+} (1{-}r_\psi) \max_{\mathfrak{a}^\times} Q_\psi(s^\times{'},\mathfrak{a}^\times))$\vspace{-4pt}
\ELSE \vspace{-1pt}
\STATE $Q_\psi(s^\times,a^\times) \leftarrow  (1-\alpha)Q_\psi(s^\times,a^\times)$
\ENDIF
\IF{$s^\times{'} {\in} B_\varphi^\times$}\vspace{-1pt}
\STATE $Q_{\psi,\varphi}(s^\times,a^\times) \leftarrow  (1{-}\alpha)Q_{\psi,\varphi}(s^\times,a^\times)$\vspace{-1pt}
\STATE ${+}\alpha (r_\varphi {+} (1{-}r_\varphi) \max_{\mathfrak{a}^\times} Q_{\psi,\varphi}(s^\times{'},\mathfrak{a}^\times))$\vspace{-4pt}
\ELSE \vspace{-1pt}
\STATE $Q_{\psi,\varphi}(s^\times,a^\times) \leftarrow  (1{-}\alpha)Q_{\psi,\varphi}(s^\times,a^\times)$ \vspace{-1pt}
\STATE $+ \alpha ((1{-}r_\varphi^2) \max_{\mathfrak{a}^\times} Q_{\psi,\varphi}(s^\times{'},\mathfrak{a}^\times))$\vspace{-1pt}
\ENDIF
\STATE $Q_{\psi,\varphi}^R(s^\times,a^\times) \leftarrow  (1{-}\alpha)Q_{\psi,\varphi}^R(s^\times,a^\times)$\vspace{-1pt}
\STATE \hspace{5pt} ${+} \alpha (r {+} \gamma Q_{\psi,\varphi}^R(s^\times{'},a^\times{'}))$\vspace{-1pt}
\end{footnotesize}
\end{algorithmic}
\end{algorithm}
\end{minipage}
\hfill
\begin{minipage}{0.36\textwidth}
\begin{algorithm}[H]
\caption{choose\_action}
\label{alg:choose_action}
\begin{algorithmic}
\itemindent=-10pt
\begin{footnotesize}
\STATE \textbf{Input:} $Q_\psi, Q_{\psi,\varphi}, Q_{\psi,\varphi}^R$ and $s^\times$ 
\STATE $u \sim \mathcal{U}(0,1)$
\IF{$u\leq \epsilon$}
\STATE \textbf{return} a random action $a^\times \in A^\times(s^\times)$
\ENDIF 
\STATE $V_\psi(s^\times) {\leftarrow} \max_{a^\times {\in} A^\times(s^\times)} Q_\psi(s^\times,\mathfrak{a}^\times)$
\STATE $\hat{A}^\times_\psi\hspace{-1pt}(\hspace{-1pt}s^\times\hspace{-1pt}) {\leftarrow}  \{ a^\times \hspace{-1pt}{\in} A^\times\hspace{-1pt}(\hspace{-1pt}s^\times\hspace{-1pt}) \mid V_\psi\hspace{-1pt}(\hspace{-1pt}s^\times\hspace{-1pt}) {-} Q_\psi\hspace{-1pt}(\hspace{-1pt}s^\times\hspace{-3pt},a^\times\hspace{-1pt})\hspace{-1pt} {\leq} \tau_\psi \}$
\STATE $V_{\psi,\varphi}(s^\times) {\leftarrow} \max_{a^\times {\in} \hat{A}_\psi^\times(s^\times)} Q_{\psi,\varphi}(s^\times,\mathfrak{a}^\times)$\vspace{-1pt}
\STATE $\hat{A}^\times_{\hspace{-1pt}\psi\hspace{-1pt},\hspace{-1pt}\varphi}\hspace{-1pt}(\hspace{-1pt}s^\times\hspace{-1pt}) {\leftarrow} \{\hspace{-1pt} a^\times \hspace{-2pt}{\in}\hspace{-1pt} \hat{A}_{\hspace{-1pt}\psi}^\times\hspace{-2pt}(\hspace{-1pt}s^\times\hspace{-1pt}) {\mid} V_{\hspace{-1pt}\psi\hspace{-1pt},\hspace{-1pt}\varphi}\hspace{-1pt}(\hspace{-1pt}s^\times\hspace{-1pt}) {-} Q_{\hspace{-1pt}\psi\hspace{-1pt},\hspace{-1pt}\varphi}\hspace{-1pt}(s^\times\hspace{-4pt},a^\times\hspace{-2pt}) {\leq} \tau_\varphi \}$
\IF{$u\leq \epsilon+\upsilon$}
\STATE \textbf{return} a random action $a^\times \in \hat{A}^\times_{\psi,\varphi}(s^\times)$
\ENDIF\vspace{-4pt}
\STATE \textbf{return} $a^\times {\in} \argmax_{\mathfrak{a}^\times {\in} \hat{A}^\times_{\psi,\varphi}(s^\times)} Q_{\varphi,\psi}^R(s^\times,a^\times)$
\end{footnotesize}
\end{algorithmic}
\end{algorithm}
\end{minipage}
\end{figure*}

\section{Case Study}
\label{sec:case}
\vspace{-2pt}
 
We implemented Algorithm~\ref{alg:main} on top of \emph{CSRL}~\cite{bozkurt2020a,bozkurt2021}. We set the discount factor $\gamma{=}0.99$, safety reward $r_\psi{=}0.0001$, and the LTL reward $r_\varphi{=}0.01$. In addition, we initialized the learning rate $\alpha$, safety threshold $\tau_\psi$,  LTL threshold $\tau_\varphi$, and the parameter~$\upsilon$ for the LTL objective to $0.5$, and gradually decreased them to $0.05$. Similarly, we slowly decreased the exploration parameter $\epsilon$ from $0.5$ to $0.005$.

\begin{figure}[!t]
  \begin{minipage}[c]{0.6\columnwidth}
    \includegraphics[width=1\columnwidth]{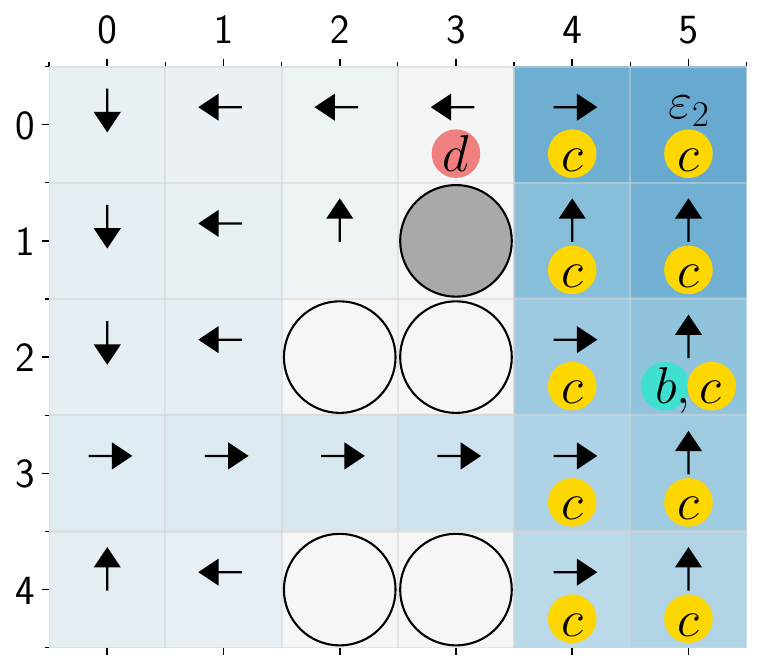}
  \end{minipage}\hfill
  \begin{minipage}[c]{0.39\columnwidth}
    \caption{The grid world with the learned policy. Empty circles: absorbing states. Filled circle: obstacle. Encircled letters: labels. Arrows: actions. Estimated values are represented by the shades of blue; the darker, the higher~value.} \label{fig:lexicographic_policy}
  \end{minipage}
  \vspace{-10pt}
\end{figure}

We conducted our numerical experiments on the $5\times6$ grid world shown in Fig.~\ref{fig:lexicographic_policy}, where each cell represents a state. The agent starts in the cell $(0,0)$ moves from one cell to one of its neighbors via four actions: \emph{up}, \emph{down}, \emph{right} and \emph{left}. When an action is taken, the agent moves in the intended direction w.p. $0.8$; and moves sideways w.p. $0.2$ ($0.1$ for each direction). If the direction is blocked by an obstacle or the boundary of the grid, the agent stays in the same cell. Lastly, if the agent moves into an absorbing state it cannot~leave.

The safety specification is $\psi=\square \neg(d \wedge \bigcirc d)$, which means, in the context of this grid, the cell $(0,3)$ should not be visited consecutively. The LTL specification is $\varphi= \square \lozenge b \wedge \lozenge \square c$, which can be interpreted as the right part of the grid must be eventually reached and must not be left after some time steps, and the cell $(2,5)$ must be repeatedly visited. Both formulas are translated to a 3-state automaton.  All the rewards are zero except for the reward of the cell $(0,5)$, which is 1.

There are two ways to reach the right part of the grid to satisfy $\varphi$: (i) through the cell labeled with $d$ ($(0,3)$) and (ii) through the cells between the absorbing states ($(3,2)$, $(3,3)$). Through (i), the right part can be reached w.p. 1 ($Pr_{\max}(s_0 {\models} \varphi){=}1$); however, the safety property  $\psi$ can be violated w.p. 0.2 ($Pr_{\max}(s_0 {\models} \psi){=}0.8$), due to the probability that the agent goes sideways and hit the boundary or the obstacle and stays in $(0,3)$ in two consecutive time steps. Through (ii), $\psi$ is guaranteed to be satisfied ($Pr_{\max}(s_0 {\models} \psi)=1$) although the probability that the agent gets stuck in one of the absorbing states is $0.36$ ($Pr_{\max}(s_0 {\models} \varphi)=0.64$). Thus, a policy prioritizing the safety chooses (ii) over (i) while a policy with a single objective $\varphi'=\psi \wedge \varphi$ prefers (i) because $Pr_{\max}(s_0 {\models} \varphi')=0.8$ through (i) and $Pr_{\max}(s_0 {\models} \varphi')=0.64$ through (ii).

The policy shown in Fig.~\ref{fig:lexicographic_policy} is learned after $128K$ ($K=2^{10}$) episodes, each with a length of $1K$. The agent learned to eliminate the actions other than \emph{left} in $(0,2)$ and the ones other than \emph{right} in $(0,4)$ for satisfaction of $\psi$. Similarly, except for $(3,2)$ and $(3,3)$, the agent prunes all the actions that might lead to an absorbing state due to $\varphi$. Since the only positive reward can be obtained in $(0,5)$, the policy aims to reach $(0,5)$ as soon as possible using the remaining actions. Once reached, the policy takes the action $\varepsilon_2$ to change the LDBA state to satisfy the B\"uchi condition. In addition, it takes a random action w.p. $\upsilon=0.05$ in $(0,5)$, $(1,5)$, $(2,5)$, $(3,5)$, $(4,5)$ and $(1,4)$, which ensures visiting $(2,5)$, the cell labeled with $b$ and $c$, infinitely many times. The maximum expected QoC return that can be obtained after visiting $(0,5)$ is about $88$, and approximately a QoC return of $85$ is obtained by following the learned policy. 

\vspace{-2pt}
\section{Conclusion}
\label{sec:conclusion}
\vspace{-4pt}
In this paper, we have introduced a model-free RL method that learns near-optimal policies for given safety, LTL and QoC objectives, where the safety objective is prioritized above the LTL objective, which is prioritized above the QoC objective. Our algorithm learns the safest actions and among those learns the actions sufficient to maximize the probability of satisfying the LTL objective. Finally, using these actions, our algorithm converges to a policy maximizing the expected return for a given parameter $\upsilon$ controlling the randomness of the obtained  policy. %
We have shown that as long as $\upsilon$ is positive, the satisfaction probability for the LTL objective is maximized and as $\upsilon$ goes to zero, the expected return approaches its supremum value. Lastly, we have demonstrated the applicability of our algorithm on a case study.

\vspace{-4pt}
\bibliography{references.bib,ref_yu.bib}
\bibliographystyle{unsrt}
\end{document}